\documentclass[twosides]{article}

\usepackage[accepted]{aistats2023}
\usepackage[round]{natbib}

\usepackage[hidelinks]{hyperref}
\usepackage{comment}
\usepackage{todonotes}
\usepackage{amsmath}
\usepackage{amsfonts}
\usepackage{amssymb}
\usepackage{natbib}
\usepackage{algpseudocode}
\usepackage{enumitem}
\usepackage{booktabs}
\usepackage[toc,page]{appendix}
\usepackage{dsfont}
\usepackage{amsthm}
\usepackage{graphicx}
\usepackage{algorithm2e}
\usepackage{url}
\usepackage{makecell}
\usepackage{multicol}
\newtheorem{theorem}{Theorem}
\newtheorem{proposition}{Proposition}
\RestyleAlgo{ruled}

\DeclareMathOperator*{\argmin}{\arg\!\min}

\begin{document}

%

%

\twocolumn[
\aistatstitle{Rethinking Sharpness-Aware Minimization as Variational Inference}

\aistatsauthor{Szilvia Ujv\'ary$^{*1}$ \And Zsigmond Telek$^{*2}$ \And  Anna Kerekes$^{*1}$ \And Anna Mesz\'aros$^{*1}$ \And Ferenc Husz\'ar$^1$}

\aistatsaddress{} ]

\begin{abstract}
Sharpness-aware minimization (SAM) aims to improve the generalisation of gradient-based learning by seeking out flat minima. In this work, we establish connections between SAM and Mean-Field Variational Inference (MFVI) of neural network parameters. We show that both these methods have interpretations as optimizing notions of flatness, and when using the reparametrisation trick, they both boil down to calculating the gradient at a perturbed version of the current mean parameter.
This thinking motivates our study of algorithms that combine or interpolate between SAM and MFVI. We evaluate the proposed variational algorithms on several benchmark datasets, and compare their performance to variants of SAM. Taking a broader perspective, our work suggests that SAM-like updates can be used as a drop-in replacement for the reparametrisation trick.
\end{abstract}

\section{Introduction}

The training loss of overparametrized models often has several equally low minima, which may differ widely in their performance at test time. The most salient effort to characterise which minima lead to better generalisation is the one that links \emph{flatness} to generalisation \citep{HochreiterSchmidhuber1997, Keskar2016, Chaudhari16, DziugaiteR17, Dinh17, Jiang21}.
Motivated by this intuition, Sharpness Aware Minimization (SAM) aims to directly bias optimization towards flat regions. Its simple and intuitive notion of flatness gives rise to an algorithm that substantially improves generalisation while only doubling the computational cost per iteration.

Other methods, such as Mean-Field Variational Inference (MFVI) and Evolution Strategies (ES) also have \emph{flatness-seeking} interpretations \citep{ES}. Indeed, the pursuit of flat minima was originally motivated by a Bayesian, minimum description length (MDL) argument \citep{HochreiterSchmidhuber1997}. Thus, it is not surprising that minimizing a variational Bayesian objective, itself an estimate of the description length of training data \citep{miracle}, favours flat minima. This connection between variational Bayesian deep learning and SAM has remained largely unexplored.

In this work we show that the connections between MFVI and SAM are more than philosophical. Concrete implementations of SAM and MFVI (with the reparametrization trick) both calculate gradients at a perturbed point around the current parameter. While SAM takes the worst-case perturbation, MFVI uses a random Gaussian perturbation. In this sense, the SAM update can be thought of as a biased deterministic approximation replacing the unbiased but high-variance single-sample Monte Carlo gradient in MFVI. Exploring this further, we  make the following contributions:

\begin{enumerate}[nosep]
    \item We establish connections between SAM and MFVI, characterising and comparing their flatness-seeking inductive biases (Section \ref{MethodAnalysis}).
    \item We propose the algorithm VariationalSAM, which combines aspects of SAM and MFVI.
    \item We compare all methods on benchmark datasets (CIFAR-10 and CIFAR-100).
\end{enumerate}





\section{Sharpness-Aware Minimization (SAM)}

\begin{table*}[h]
    \centering
    \begin{tabular}{l|c|cc|c}
        Name & perturbation & covariance & $\Sigma=$ & Penalty \\
        \hline
        \hline
        SAM\,\citep{SAM} & worst-case & fixed & $\frac{\rho^2}{p}I$ & $\operatorname{KL =}L_2$ \\
        \hline
        Random SAM (MFVI $\mu$ only) & Gaussian & fixed & $\frac{\rho^2}{p}I$ & $\operatorname{KL}$ \\
        \hline
        MFVI \citep{HintonvanCamp} & Gaussian & learned & $\operatorname{diag}(\sigma_i)$ & $\operatorname{KL}$ \\
        \hline
        Variational SAM & worst-case & learned & $\operatorname{diag}(\sigma_i)$ & $\operatorname{KL}$\\
        \hline
        Adaptive SAM\,\citep{ASAM} & worst-case & $\mu$-adaptive &   $\operatorname{diag}(\vert\frac{1}{\mu_i}\vert)$ & $L_2$ \\
        \hline
        Fisher SAM\,\citep{FSAM} & worst-case & $\mu$-adaptive &  $\operatorname{diag}(F(\mu))$ & $L_2$ \\
        \hline
        \makecell[l]{Evolution Strategy\,(ES)\, \\ \citep{ES, Beyer2002} } & Gaussian & fixed & $\frac{\rho^2}{p}I$ & none \\
        \hline
        \makecell[l]{CMA-ES\,\citep{CMAES}, \\ VO\,\citep{VO}, \\ NES\,\citep{NES}} & Gaussian & learned & $\operatorname{diag}(\sigma_i)$ & none \\
        \end{tabular}
    \caption{An overview of methods mentioned in this work. All methods can be related to MFVI or SAM by changing the perturbation type, shape of $\Sigma$, or penalty terms.}
    \label{tab:methods}
\end{table*}

In this section, we present SAM with a slight change in notation relative to original works in order to make connections to variational methods more apparent later. Denoting the parameter by $\mu$, and the number of parameters by $p$, the idealized loss function SAM aims to minimize can be written as

\begin{flalign}
    \begin{aligned}
       L_{\text{SAM}} (\mu, \Sigma) &=  \max_{\epsilon^T \Sigma^{-1} \epsilon\leq p} \left[L(\mu + \epsilon)-L(\mu)\right] \\
       &+ L(\mu) + \alpha \|\mu\|_2^2\label{eqn:SAM_ideal} 
    \end{aligned}
\end{flalign}

Here, $L$ is the average training loss, $\Sigma$ is any symmetric positive definite matrix. Setting $\Sigma = \frac{\rho^2}{p}I$, where $\rho$ is a constant scalar parameter recovers the standard SAM \citep{SAM}, while setting $\Sigma$ adaptively as a function of $\mu$ recovers newer variants Adaptive SAM (ASAM) and Fisher SAM (FSAM) \citep{ASAM, FSAM}. The parameter $\alpha>0$ controls the strength of $L_2$ regularization. 

Computing gradients of (Eqn.\  \eqref{eqn:SAM_ideal}) with respect to $\mu$ is intractable in general, hence in practice an approximation is used. Taking a first-order Taylor expansion of $L$ we can approximate the optimal $\epsilon$ as follows \citep{SAM, ASAM, FSAM}:

\begin{equation}
    \epsilon^* (\mu)= \sqrt{p} \frac{\Sigma \nabla^\top_\theta L(\mu)}{\sqrt{\nabla^\top_\theta L(\mu) \Sigma \nabla_\theta L(\mu)}}
    \label{epsilon}
\end{equation}

where $\nabla^\top_\theta L(\mu)$ denotes the gradient of the loss evaluated at $\mu$. For reasons that will be clear later, we introduce new notation $\eta(\mu)$ and re-express $\epsilon^*$ from $\eta$ as follows:
\begin{equation}
    \epsilon^*(\mu) = \Sigma^\frac{1}{2} \eta(\mu),
    \label{reparam}
\end{equation}
\begin{align*}
     \text{where }\eta(\mu) = \frac{\sqrt{p}\tilde{g}(\mu)}{\|\tilde{g}(\mu)\|_2}
     \text{ and }
     \tilde{g}(\mu) = \Sigma^\frac{1}{2}\nabla^\top_{\theta}L(\mu).
\end{align*}

The SAM update direction is obtained by plugging this estimate $\epsilon^*$ back into (Eqn.\ \eqref{eqn:SAM_ideal}), ignoring the dependence of $\eta$ on $\mu$, and differentiating the resulting expression. This gives the following update direction:

\begin{equation}
\nabla^\top_\mu L_\text{SAM} \approx \nabla^\top_\theta L(\mu + \Sigma^\frac{1}{2}\eta(\mu)) + 2\alpha\mu
\label{eqn:SAM_actual}
\end{equation}

Note that computing this update direction requires backpropagation twice. First, we calculate the gradient at $\mu$ to obtain the worst-case perturbation $\eta(\mu)$. Second to have to evaluate the gradient $\nabla^\top_\theta L$ at the perturbed location $\mu+\Sigma^\frac{1}{2}$. Thus, the computational cost of a single SAM update is about twice that of typical SGD. We note that the SAM update can be used in conjunction with any stochastic optimizer including Adam \citep{Adam}.

\section{Mean-Field Variational Inference (MFVI)\label{sec:MFVI}}

We can also take an approximate Bayesian approach to inferring model parameters from data. In this view, we start with a prior distribution, $p(\theta)$ over model parameters $\theta$, which we choose to be a Gaussian with mean $0$ and standard deviation $\sigma_0$.

We then attempt to infer the posterior distribution of weights using the Bayes' rule

\begin{equation*}
p(\theta\vert \mathcal{D}) \propto e^{-NL(\theta)}\cdot p(\theta),
\end{equation*}
where $N$ denotes the number of data points and we assumed that the loss function $L$ calculates the average negative log likelihood of i.\,i\,.d.\ training data. This calculation being intractable, we try to approximate this by an approximate posterior $q(\theta)$, chosen to be Gaussian with mean $\mu$ and covariance $\Sigma$.

Minimizing the KL-divergence between this approximate $q$ and the true posterior yields the following objective function for $\mu$ and $\Sigma$  \citep{HintonvanCamp}:
\begin{flalign}
    \begin{aligned}
        L_\text{MFVI}(\mu, \Sigma) &= \mathbb{E}_{\theta \sim \mathcal{N}(\mu, \Sigma)} L(\theta) \\
        &+ \frac{1}{N}\operatorname{KL}\left[\mathcal{N}(\mu, \Sigma)\| \mathcal{N}(0, \sigma^2_0I)\right]
    \label{eqn:MFVI_ideal}
    \end{aligned} 
\end{flalign}

\begin{table*}[h]
\centering
\begin{tabular}{l|l|l}
                                                    & SAM                          & MFVI                                               \\ 
\hline
\makecell[l]{Analysis of idealized loss\\(what it says on the box)}          & $L(\mu) + \sqrt{p}||g(\mu)||_\Sigma$ & $L(\mu) + \frac{1}{2}\text{Tr}[\Sigma H(\mu)]$               \\
\hline
\makecell[l]{Analysis SGD dynamics\\(what the algorithm does)} & $L(\mu) + \sqrt{p} ||g(\mu)||_\Sigma$ + $\frac{\delta}{4} ||g(\mu)||_2^2$  & $L(\mu)  + \frac{\delta}{4}\text{Tr}[\Sigma H(\mu)^2] + \frac{\delta}{4}||g(\mu)||_2^2$ \\
\hline
\end{tabular}
\caption{Summary of the flatness-seeking regularization properties of SAM and MFVI. In the \textbf{first row} we approximate the idealized loss functions of SAM (Eqn.\ \eqref{eqn:SAM_ideal}) and MFVI (Eqn.\ \eqref{eqn:MFVI_ideal}) in terms of higher order derivatives. SAM approximately penalizes the norm of the gradient, while MFVI penalizes the trace of the Hessian. Near minima, $\text{Tr}[\Sigma H(\mu)]$ is a good measure of sharpness, but it can take negative values around saddle points and local maxima. In the \textbf{second row} we approximate the implicit regularization properties of stochastic gradient descent with SAM (Eqn.\ \eqref{eqn:SAM_actual}) or MFVI gradients (Eqn.\ \eqref{eqn:MFVI_actual}), with small but finite learning rate $\delta$. The $\frac{\delta}{4}\|g(\mu)\|$ term represents the implicit bias of GD with finite step size as in as in \citep{Smith2021}. SAM's flatness penalty is the same as the idealized, but for MFVI the penalty contains the trace of the Hessian squared, which is now always positive.}
\label{tab:regular}
\end{table*}

When $\Sigma$ is chosen to be diagonal, the minimization of (Eqn.\ \eqref{eqn:MFVI_ideal}) is called mean-field variational inference (MFVI). Here, we liberally extend this name to algorithms where $\Sigma$ is non-diagonal. Dropping the $\operatorname{KL}$ term from (Eqn.\ \eqref{eqn:MFVI_ideal}) yields variational optimization (VO) and evolution strategies (ES) \citep{VO,CMAES, ES, NES}.

We can estimate the gradients of $L_\text{MFVI}$ using the reparametrization trick 
\citep{reparametrizationtrick}. We also expand the KL divergence term, and substitute $\alpha := \frac{1}{2N\sigma^2_0}$ to obtain:
\begin{equation}
    \nabla^\top_{\mu} L_\text{MFVI}(\mu, \Sigma) =  \mathbb{E}_{\eta \sim \mathcal{N}(0, I)} \nabla^\top_{\theta} L(\mu + \Sigma^{\frac{1}{2}} \eta) + 2\alpha\mu.
\end{equation}
We approximate this by a single-sample Monte Carlo estimate and get the update direction:
\begin{flalign}
\begin{aligned}
    \nabla^\top_{\mu} L_\text{MFVI}(\mu, \Sigma) \approx \nabla^\top_{\theta} L(\mu + \Sigma^{\frac{1}{2}} \eta) + 2\alpha\mu, \\ 
    \text{ where }\eta\sim\mathcal{N}(0, I)
    \label{eqn:MFVI_actual}
\end{aligned}
\end{flalign}

Note the similarity between Eqns.\ \eqref{eqn:SAM_actual} and \eqref{eqn:MFVI_actual}. They are of precisely the same form except for the perturbation $\eta$: in SAM, $\eta$ is chosen to be the worst-case perturbation calculated deterministically from $\mu$ by taking a gradient, while in MFVI, $\eta$ is drawn from a standard normal. Throughout this work, we will refer to the variant of MFVI where we fix $\Sigma=\frac{\rho^2}{p}I$ as \emph{RandomSAM}, or $L_\text{RSAM}(\mu)$.

\section{Exploring the relationship}

This section aims to further strengthen the connection mentioned above. We first establish that the SAM objective upper bounds the MFVI objective, and discuss what this means in terms of the theoretical justification of both algorithms. Then, we discuss connections between flatness and the minimum description length paradigm.

\subsection{SAM as an upper bound on VI}

The similarity between the SAM and MFVI updates is not surprising given that $L_{SAM}$ can be considered as a loose upper bound on $L_{MFVI}$ when the variance $\Sigma$ is sufficiently small, and the number of parameters, $p$ is sufficiently large. This is because in high dimensions, samples from  $\mathcal{N}(\mu, \Sigma)$ concentrate around the ellipsoid $(x-\mu)^\top\Sigma^{-1}(x - \mu) = p$. Thus, any expectation over the $\mathcal{N}(\mu, \Sigma)$ can be upper bounded by the maximum value within the ellipsoid i.\,e.\ $\max_{(x-\mu)^\top\Sigma^{-1}(x - \mu) \leq p} L(\theta)$.

Stated informally, in high dimensions the following relationship holds between the MFVI and SAM objectives:

\begin{equation*}
    \mathbb{E}_{\theta \sim \mathcal{N}(\mu, \Sigma)} L(\theta) \lessapprox \max_{\epsilon^T \Sigma^{-1} \epsilon\leq p} L(\mu + \epsilon)
\end{equation*}

This upper bound relationship is exploited in the theorems used to justify SAM \citep{SAM}, ASAM \citep{ASAM} and FSAM  \citep{FSAM}. In these prior work, proofs are given that the SAM objective upper bounds test error under certain assumptions. However, it is only the last step of these proofs where the upper bound relationship is exploited. Thus the same proofs, based mostly on PAC-Bayes, provide even stronger bounds and guarantees for the MFVI objective. Therefore, if the theoretical justification provided for SAM indeed meaningfully explains its empirical success, one could expect similar or better performance from MFVI algorithms as well.

\subsection{Flatness and description length}

Flatness and minimum description length are closely related concepts. The pursuit of flat minima in deep learning was originally motivated by a minimum description length argument \citep{HochreiterSchmidhuber1997}. The MFVI objective, too, is typically motivated as minimizing the description length of training data \citep{Hinton1993KeepingTN, miracle}. Although not usually presented this way, MFVI can also be seen to minimizing the following sharpness penalty:
\begin{equation*}
    \mathbb{E}_{\theta \sim \mathcal{N}(\mu, \Sigma)} [L(\theta) - L(\mu)]
\end{equation*}
Unlike SAM's sharpness penalty, however, this penalty is not necessarily positive. Indeed can take negative values around saddle points and local maxima, locations of the parameter space we would prefer to avoid.

In the following section we will look at the implicit sharpness-avoiding bias of both SAM and MFVI, attempting to characterize them precisely in terms of higher order derivatives.

\section{Results on implicit regularization}
\label{MethodAnalysis}

Here, we characterize the implicit regularization towards flat minima in terms of higher order derivatives, borrowing techniques from \cite{Roberts2018, Smith2021, gdregul}. Some of these techniques approximate the loss function, while others give an approximation on the gradient descent path. The latter will explain better, how these algorithms work in practice. We summarise these results in Table (\ref{tab:regular}) and give further explanations in the following four propositions (we set $\alpha=0$ ignoring the $L_2$ or $\operatorname{KL}$ penalties for readability):
\begin{proposition}
    \label{InductiveBias}
    The following approximation holds for the SAM objective ($\Sigma=\frac{\rho^2}{p}I$):
    \begin{align}
        L_{\text{SAM}} (\mu) \approx L(\mu) + \rho \|\nabla_\theta L(\mu)\|_2.
        \label{eqn:SAM_inductive}
    \end{align}
        Furthermore, for general $\Sigma$, this takes the form
    \begin{flalign}
        \begin{aligned}
            L_{\text{SAM}} (\mu, \Sigma) &\approx L(\mu) + \sqrt{p}\|\nabla_\theta L(\mu)\|_\Sigma \\
            &= L(\mu) + \sqrt{p}\sqrt{\nabla^\top_\theta L(\mu) \Sigma \nabla_\theta L(\mu)} .
        \end{aligned}
    \end{flalign}
\end{proposition}

\begin{proof}[Proof (sketch)]
Performing a first-order Taylor approximation, and plugging in $\epsilon^*$ from (Eqn.\ \eqref{epsilon}), we arrive at
\begin{flalign}
\begin{aligned}
L_{\text{SAM}}(\mu, \Sigma)
&= L(\mu) + \sqrt{p}\sqrt{\nabla^\top_\theta L(\mu) \Sigma \nabla_\theta L(\mu)}  \\
&+ O(\epsilon^{*\top}  \epsilon^*).
\end{aligned}
\end{flalign}
By using the reparametrization (Eqn.\ \eqref{reparam}), we can bound the error by $O(\rho^2)$. For details, see Appendix \ref{ImplicitRegularization}.
\end{proof}
Thus we can think of SAM as a way to penalize the gradient norm. In fact, two different results show that this interpretation holds both for the ideal $L_\text{SAM}$ from Eqn.\ \eqref{eqn:SAM_ideal}, and for the approximate SAM step from Eqn.\ \eqref{eqn:SAM_actual} (see Appendix \ref{ImplicitRegularization} for details). This penalty is similar to the squared gradient norm found to describe the implicit regularization of full batch gradient descent in \cite{gdregul}, which is shown for SAM in the following Proposition \ref{sam_bwe}.
\begin{proposition}
\label{sam_bwe}
Gradient descent with the SAM  step follows a path that is closest to the exact continuous path given by $\dot{\mu}=-\nabla_{\mu}\tilde{L}_{\text{SAM}}(\mu)$, where $\tilde{L}_{\text{SAM}}(\mu)$ is given by
\begin{flalign}
    \begin{aligned}
    \label{prop2}
        \tilde{L}_{\text{SAM}}(\mu) &\approx L(\mu) + \sqrt{p} \|\nabla_\theta L(\mu)\|_\Sigma\\
        &+\frac{\delta}{4}||\nabla_\theta L(\mu) ||_2^2,
    \end{aligned}
\end{flalign}
where $\delta$ is the stepsize of the gradient descent algorithm.
\end{proposition}
\begin{proof}[Proof (sketch)]
As in \cite{gdregul} we get the following:
\begin{align}
    \tilde{L}_\text{SAM}(\mu)\approx L_\text{SAM}(\mu) + \frac{\delta}{4} ||\nabla_\theta L_\text{SAM}(\mu)||_2^2
\end{align}
After Taylor's approximation, we obtain (Eqn.\  \eqref{prop2}).
\end{proof}
This approximation confirms the \emph{flatness-seeking} behaviour of the SAM algorithm. It is similar to the previous result, with an extra regularizer term, which penalizes sharpness.

\begin{proposition}
    \label{sharpness}
    The following approximation holds for the RandomSAM algorithm with constant  $\Sigma=\frac{\rho^2}{p}I$:
    \begin{align}
        L_{\text{RSAM}}(\mu) 
        \approx L(\mu) + \frac{\rho^2}{2p} \text{Tr} H(\mu)
    \end{align}
where $H(\mu)$ denotes the Hessian. For MFVI with general $\Sigma$ we have the following approximation:
    \begin{align}
        L_{\text{MFVI}}(\mu, \Sigma) 
        \approx L(\mu) + \frac{\text{Tr}\left[ \Sigma H(\mu)\right]}{2}
    \end{align}
\end{proposition}
\begin{proof}[Proof (sketch)]
The MFVI objective is
\begin{align}
    L_{\text{MFVI}}(\mu)= \mathbb{E}_{\epsilon \sim \mathcal{N}(0, I)} L(\mu+\Sigma^{\frac{1}{2}}\epsilon),
    \label{eqn:vo_eq}
\end{align}
Again, a first-order Taylor approximation gives
\begin{align}
L_{\text{MFVI}}(\mu) \approx L(\mu) + \frac{1}{2}\text{Tr} \big[\Sigma H(\mu) \big].
\end{align}
\end{proof}
A way of defining sharpness is via the local curvature of the loss function around the minimum given that it is a critical point \citep{Dinh17}. By Proposition \ref{sharpness}, MFVI implicitly regularizes the trace and hence around an optimum its eigenvalues. Since local curvature is encoded in the Hessian eigenvalues,  this means that MFVI penalizes a notion of sharpness at critical points of the loss landscape.

\begin{figure*}[t]
    \centering
    \includegraphics[trim=5 0 0 0, clip, width=1.0\textwidth]{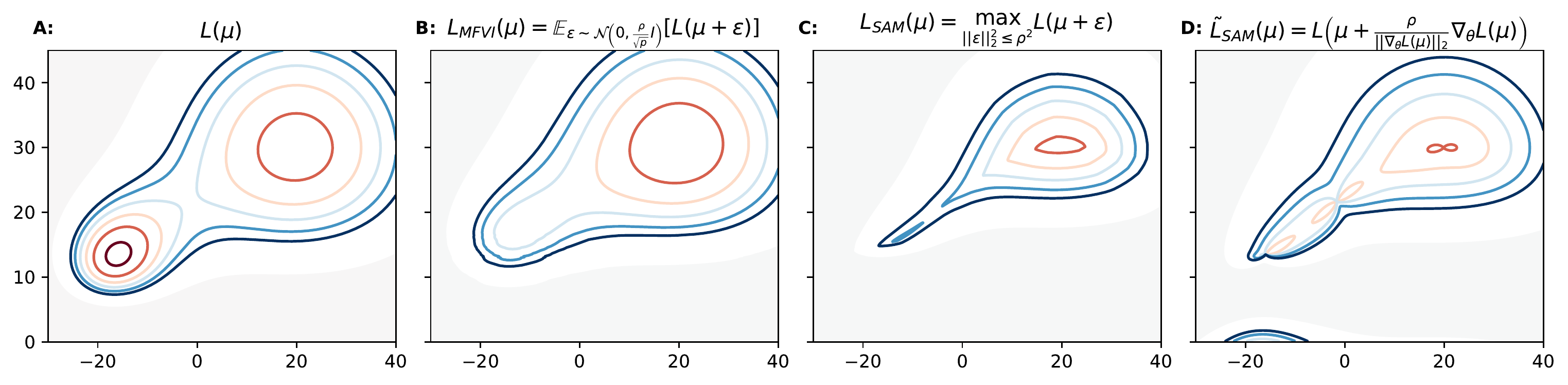}
    \caption{Illustration of how MFVI and SAM encourage convergence to flatter minima on a 2D toy example. \textbf{A}: The original loss $L$, taken from \cite{FSAM}, has two minima, a sharp and a wide one. \textbf{B: }in MFVI, averaging over the Gaussian variational posterior smoothes out the sharper minimum, increasing the attraction basin of the flat one. \textbf{C: }considering the worst-case within a Euclidean ball achieves a similar transformation \textbf{D: }the SAM update relies on a Taylor approximation, which does not apply in highly non-linear regions (in all cases $\rho=8$
    \label{fig:2d_toy}).}
\end{figure*}

\begin{proposition}
    \label{mfvi_bwe}
    Gradient descent with MFVI step follows a path that is closest to the gradient flow path on $\tilde{L}_\text{MFVI}(\mu, \Sigma)$, where $\tilde{L}_\text{MFVI}(\mu, \Sigma)$ is the following
    \begin{align}
        &\tilde{L}_\text{MFVI}(\mu,\Sigma) \approx  L(\mu) + \frac{\delta}{4} ||g(\mu)||_2^2 +\frac{\delta}{4} \text{Tr}[\Sigma H(\mu)^2],
    \end{align}
    where $g(\mu + \Sigma^\frac{1}{2}\eta)=\nabla_\theta L(\mu + \Sigma^\frac{1}{2} \eta).$

\end{proposition}

\begin{proof}[Proof (sketch)]

Since the MFVI objective is not typically available in closed form, in practice one uses a single-sample Monte Carlo estimate based on the reparametrization trick. This yields a stochastic objective as follows.
\begin{align}
    L_{\text{MFVI}}(\mu) \approx L(\mu + \Sigma^{\frac{1}{2}}\eta), \quad  \eta \sim \mathcal{N}(0, I)
\end{align}
Now by \cite{gdregul} and taking expectation we get
\begin{flalign}
    \begin{aligned}
        &\tilde{L}_\text{MFVI}(\mu,\Sigma)\approx \\
        \mathbb{E}_{\eta\sim\mathcal{N}(0,I)}&\left[ L(\mu + \Sigma^{\frac{1}{2}}\eta) + \frac{\delta}{4}||\nabla_\theta L(\mu+\Sigma^{\frac{1}{2}}\eta)||_2^2 \right]
    \end{aligned}
\end{flalign}
We can finish up the proof with a Taylor approximation.
\end{proof}

This shows that Mean Field Variational Inference implicitly regularizes the trace of the Hessians square, and hence its eigenvalues, which are a meaningful metric of sharpness. However, note that the strength of this penalty appears to be much smaller, in the order of $\frac{\rho^2}{p}$ rather than $\rho$. The full proofs can be found in Appendix \ref{VOSharpness}.

In Figure \ref{fig:2d_toy} we illustrate the sharpness-avoidance of $L_\text{SAM}$ and $L_\text{RSAM}$ on a 2D toy example originally introduced in \cite{FSAM}. We can see that both transformations are effective at reducing the attraction basin of the sharp minimum. What is less clear, however, is which one of these translates better to the high-dimensional problems we encounter in deep learning, and indeed which one leads to better generalization.

The above connections between MFVI and SAM motivate our two main research questions which we investigate in the rest of this report:

\begin{enumerate}[nosep]
    \item All things being equal, is SAM or \texttt{RandomSAM} (MFVI with fixed $\Sigma=\frac{\rho^2}{p}I$) more effective at finding minima that generalises well?
    \item It is possible to learn $\Sigma$ in SAM via gradient descent along with $\mu$ as in MFVI. This results in an algorithm we call \texttt{VariationalSAM} (VSAM). Does this have advantages over SAM where $\Sigma$ is fixed?
\end{enumerate} 


\subsection{Algorithms}
This section provides further descriptions of the variational algorithms mentioned in the previous sections: Mean-Field Variational Inference, RandomSAM, and VariationalSAM.

\begin{algorithm}[h]
\label{MeanFieldAlg}
 \SetAlgoLined
 \caption{Mean Field Variational Inference Algorithm}
 \SetKwInOut{Input}{Input}
 \SetKwInOut{Initialize}{Initialize}
 
 \Input{Training set $S=\{(x_i,y_i)\}$, parameter $\sigma_0$}
 \Initialize{$\Sigma$ and $\mu$}
 \For{$t=1,2, \cdots $}{ ~\\
 (1) Sample batch $B\sim S$ ~\\
 (2) Take a sample of $\eta \sim N(0,I)$ ~\\
 (3) Compute the gradient of the loss of $\mu$ on batch B, i.e. $\left[\nabla_{\mu}L_B\right]_{\mu + \Sigma^{\frac{1}{2}}\eta}$ ~\\
 (4) Compute the loss of $\Sigma$ on batch B, i.e. $\left[\nabla_\Sigma L_B\right]_{\mu + \Sigma ^{\frac{1}{2}} \eta } + \nabla_\Sigma \text{KL} [(\mu,\Sigma)||(0,\sigma_0 I)]$ ~\\
 (5) Update $\mu$ and $\Sigma$, i.e. ~\\
 $\mu \longleftarrow \mu - \eta_1 \left[\nabla_{\mu}L_B\right]_{\mu + \Sigma^{\frac{1}{2}}\eta}$ ~\\
 $\Sigma \longleftarrow \Sigma - \eta_2 \left(\nabla_{\Sigma} \left[ L_B\right]_{\mu + \Sigma^{\frac{1}{2}} \eta } +\nabla_\Sigma \text{KL}[(\mu,\Sigma)||(0, \sigma_0 I)]\right)$ ~\\
 }
\end{algorithm}

\texttt{RandomSAM}, introduced in Section \ref{sec:MFVI}, replaces the gradient-based first step of SAM by a random search direction. It implements MFVI with a fixed covariance $\Sigma=\frac{\rho^2}{p}I$.

In \texttt{VariationalSAM} (VSAM), we learn both $\Sigma$ and $\mu$ via gradient descent, and use the following loss:
\begin{flalign}
    \begin{aligned}
        L_{\text{VSAM}}(\mu, \Sigma) &= \max_{\epsilon^T \Sigma^{-1} \epsilon\leq p} L(\mu + \epsilon) \\
        &+ \text{KL}[\mathcal{N}(\mu, \Sigma)||\mathcal{N}(0,\sigma_0^2 I)]
    \end{aligned}
\end{flalign}
Note that this is the same as $L_\text{SAM}(\mu,\Sigma)$, but we reinterpret the $L_2$ penalty on $\mu$ as being part of a KL divergence as in MFVI. From the perspective of updating $\mu$, the $L_2$ and $\operatorname{KL}$ penalties are equivalent, but using the full $\operatorname{KL}$ allows us to also control the size of $\Sigma$. Details can be found in Appendix \ref{VSAMAlg} and \ref{VSAMMot}. In addition, we have extended the PAC-Bayes bound from \cite{SAM, ASAM, FSAM} to Variational SAM, details of the proof can be found in Appendix \ref{Pac}. Table \ref{tab:methods} summarises the relationships between methods mentioned in this report.

\begin{table*}[h]
\centering
\begin{tabular}{l|rr|rr} 
\cline{2-5}
\multicolumn{1}{l}{} & \multicolumn{2}{c|}{CIFAR-10}                                               & \multicolumn{2}{c}{CIFAR-100}                                               \\ 
\cline{2-5}
\multicolumn{1}{l}{} & \multicolumn{1}{l}{WideResNet 28-2} & \multicolumn{1}{l|}{WideResNet 28-10} & \multicolumn{1}{l}{WideResNet 28-2} & \multicolumn{1}{l}{WideResNet 28-10}  \\ 
\hline
SGD                  & $95.90^{\pm0.07}$                   & $96.97^{\pm0.12}$                     & $74.32^{\pm0.12}$                   & $80.23^{\pm0.0040}$                                     \\
SAM                  & $96.10^{\pm0.11}$                  & $97.20^{\pm0.07}$                     & $76.25^{\pm0.27}$                   & $83.26^{\pm0.0004}$                                     \\
VSAM                 & $94.18^{\pm0.11}$                   & $96.95^{\pm0.42}$                     & $74.74^{\pm0.29}$                   & $81.72^{\pm0.0200}$                                     \\
RSAM                 & $96.08^{\pm0.13}$                   & $97.11^{\pm0.05}$                     & $75.78^{\pm0.05}$                   & $80.58^{\pm0.0006}$                                     \\
\hline
\end{tabular}
\caption{Image Classification results on CIFAR-$\{10, 100\}$ using WideResNet models (VSAM=VariationalSAM, RSAM=RandomSAM). SAM performs best in all tasks, but RandomSAM gives encouraging results on CIFAR-10.}
\label{table:results}
\end{table*}

\begin{algorithm}[h]
\label{RSAMAlg}
 \SetAlgoLined
 \caption{Random SAM Algorithm - Variational Optimization}
 \SetKwInOut{Input}{Input}
 \SetKwInOut{Initialize}{Initialize}
 
 \Input{Training set $S=\{(x_i,y_i)\}$, parameter $\sigma_0$}
 \Initialize{$\mu$}
 \For{$t=1,2, \cdots $}{ ~\\
 (1) Sample batch $B\sim S$ ~\\
 (2) Take a sample of $\eta \sim N(0,I)$ ~\\
 (3) Compute the gradient of the loss of $\mu$ on batch B, i.e. $\left[\nabla_{\mu}L_B\right]_{\mu +  \sigma_0\eta}$ ~\\
 (4) Update $\mu$, i.e. ~\\
 $\mu \longleftarrow \mu - \eta_1 \left[\nabla_{\mu}L_B\right]_{\mu + \sigma_o\eta}$ ~\\
 }
\end{algorithm}

\paragraph{PAC-Bayes Bound}

We can use the same method as in \cite{FSAM} to derive a bound on the generalization error of VariationalSAM. The bound can be found below. The proof and more details can be found in Appendix \ref{Pac}.
\begin{theorem}
For a parameter space $\mathcal{M}\time\Theta$ (with later described properties) and for any $(\mu,\Sigma)\in \mathcal{M}\times\mathcal{N}$ we have with probability at least $1-\delta$, that
\begin{flalign}
    \begin{aligned}
        \mathds{E}_{\epsilon\sim \mathcal{N}(0,\Sigma)}[L_D(\mu+\epsilon)]
        &\leq \text{max}_{\epsilon^T\Sigma^{-1}\epsilon\leq \gamma^2} L_S(\mu+\epsilon)\\
        &+ \sqrt{\frac{O(p+\log\frac{n}{\delta})}{n-1}}
    \end{aligned}
\end{flalign}
where $L_D$ is the generalization loss, and $ \text{max}_{\epsilon^T\Sigma^{-1}\epsilon\leq \gamma^2} L_S(\mu+\epsilon)$ is the empirical SAM loss, with the geometry provided by $\Sigma$, i.e. the VSAM loss without the KL divergence term and $\gamma=\sqrt{p}(1+\sqrt{\log n/p})$.

\end{theorem}

\begin{algorithm}[h]
\label{VSAMAlg}
 \SetAlgoLined
 \caption{VariationalSAM Algorithm}
 \SetKwInOut{Input}{Input}
 \SetKwInOut{Initialize}{Initialize}
 
  \Input{Training set $S=\{ (x_i, y_i)\}$, parameters $\alpha$ and $\beta$, ~\\ learning rates $\eta_1$ and $\eta_2$.}
  \Initialize{$\Sigma$ and $\mu$}
  \For{$t=1,2,\cdots$}{ ~\\
    (1) Sample batch $B\sim S$ ~\\
    (2) Compute the gradient of the loss on batch B, i.e. $\nabla_{\mu} L_B (\mu)$ ~\\
    (3) Compute $\epsilon^*_{\text{VSAM}}(\mu)$ using (\ref{epsilon}) ~\\
    (4) Compute gradient approximation for the VSAM loss,~\\
    i.e. $\nabla_\mu L_{\text{VSAM}}(\mu,\Sigma)\approx\frac{\partial L}{\partial\mu}|_{\mu + \epsilon^*_{\text{VSAM}}(\mu)}$ ~\\
    (5) Update $\mu \longleftarrow \mu - \eta_1 \frac{\partial L}{\partial \mu}|_{\mu + \epsilon^*_\text{VSAM}}$ ~\\
    (6) Compute the gradient of VariationalSAM loss on batch B, i.e. $\nabla_\Sigma L_{\text{VSAM}}(\mu,\Sigma)$  ~\\
    (7) Update $\Sigma\longleftarrow \Sigma-\eta_2 \nabla_\Sigma L_{\text{VSAM}}(\mu,\Sigma)$
  }
\end{algorithm}


\section{Image Classification Experiments}

In this section, we empirically assess the generalisation performance of the mentioned algorithms: vanilla SGD, SAM, RandomSAM (RSAM) and VariationalSAM (VSAM). Our implementations of these methods are open source, and can be accessed in this \href{https://anonymous.4open.science/r/Rethinking-SAM-ANON01/README.md}{anonymized repository}.

We use WideResNets \citep{WRN} on the CIFAR-10/100 datasets \citep{cifar}. 
Following prior work of \cite{SAM, ASAM, FSAM}, we calibrate the SGD optimiser with momentum $0.9$, weight decay $0.0005$ and initial learning rate $0.1$. We use stepwise decreasing learning rate scheduling as we have found this more effective than cosine learning rate scheduling. Using batch size $128$, we train the optimizers requiring two backpropagations per step (SAM, VSAM) for up to $200$ epochs, while those with only one backpropagation (SGD and RandomSAM) are trained for up to $400$ epochs.
For CIFAR-10, we employ label smoothing \citep{label_smoothing} with factor $0.1$.

We follow \cite{SAM} in setting $\rho=0.05$ in the SAM optimizer. Random SAM and VariationalSAM require the standard deviation $\sigma$ and $\Sigma$ to be specified. In order to coincide with SAM, we have set $\sigma=\frac{\rho}{\sqrt{p}}$ $\Sigma=\frac{\rho}{\sqrt{p}} I $, where $p$ is the number of parameters of the model. In VariationalSAM, we used a learning rate of $0.01$ on $\Sigma$. We set values for $\rho$ and the penalty coefficients in order to match the KL-divergence of the prior and posterior distributions. For details see Appendix \ref{VSAMMot}.
The results are summarized in Table \ref{table:results}.

SAM visibly outperforms all other methods in all experiments. However, on CIFAR-10, RandomSAM performs almost as well as SAM. Given that RandomSAM is a much simpler method operating with random noise and only one backpropagation, this result somewhat weakens the advantages of SAM. It raises the question whether the superior results of SAM on larger datasets could be reproduced by other methods using random perturbations in the parameter space. On CIFAR-100 we see a significant gap in performance, however this may be due to our limited computing budget not allowing for more extensive grid search of parameters.

Our adaptive version of SAM, Variational SAM, does not appear to work better than SAM. This may be because during training, $\Sigma$ starts to increase, and often reaches a magnitude where the Taylor approximation underlying SAM no longer holds. This raises the question whether random MFVI would have an advantage over SAM when more flexible covariance structures are used.

\begin{table}[h]
\centering
\begin{tabular}{l|ll|}
     & CIFAR-10           & CIFAR-100           \\ 
\hline
ASAM & $96.17^{\pm 0.07}$ & $65.31^{\pm 0.11}$  \\
FSAM & $96.23^{\pm 0.06}$ & $76.54^{\pm 0.21}$  \\
\hline
\end{tabular}
\caption{Results for Adaptive SAM and Fisher SAM.\\
The network is WideResNet 28-2.}
\label{tab:ASAMFSAM}
\end{table}

Although the evaluation of Adaptive SAM and Fisher SAM are not in the focus of this paper, in Table \ref{tab:ASAMFSAM} we have included a limited set of results for comparison. For these experiments, we have set additional hyperparameters to the value reported in \cite{ASAM} and \cite{FSAM}. Namely, in ASAM we use $\gamma=0.5, \eta=0.01$ for CIFAR-10, $\gamma=1.0, \eta=0.1$ for CIFAR-100 and in FSAM $\gamma=0.1, \eta=1.0$ for both datasets.

\section{Summary and Future Work}
In this work, we have provided a novel interpretation of SAM, from the angle of Variational Inference. This led to the comparison of SAM with methods taken from Variational Inference: RandomSAM (Variational Optimization) and Variational SAM. The latter method performed rather unpromisingly. Interestingly, RandomSAM performed similarly to SAM on the CIFAR-10 dataset. This raises questions about optimality of SAM in biasing the optimisation against sharp minima, which we would like to explore in the future. 

One idea is to investigate whether some randomness could improve the generalization performance of an algorithm like SAM. This could be done by constructing an algorithm that interpolates between Random SAM and SAM. The amount used from each algorithm in the training could be a further hyperparameter.

\noindent Having established the link between SAM and Variational Bayes opens the possibility to use the SAM step in variational methods beyond those explored in this paper. For example, it would be interesting to replace the random noise by the deterministic SAM noise in the Variational Autoencoder \citep{reparametrizationtrick}. 

\noindent In most experiments, a limitation of computational power has prevented extensive grid-searches on the hyperparameters in our models. Therefore, the reported test accuracies can likely be improved with further parameter calibration.


\bibliography{sample}

\onecolumn
\setcounter{proposition}{0}
\setcounter{theorem}{0}
\newpage
\section{Appendix}
\addtocontents{toc}{\protect\setcounter{tocdepth}{0}}
\renewcommand{\thesubsection}{\Alph{subsection}}
\subsection{Implicit regularization of SAM}
\label{ImplicitRegularization}
\begin{proposition}
\label{firstprop}
    The following approximation holds for the SAM objective ($\Sigma=\frac{\rho^2}{p}I$)
    \begin{align}
        L_{\text{SAM}} (\mu) \approx L(\mu) + \rho \|\nabla_\theta L(\mu)\|_2.
    \end{align}
        Furthermore, for general $\Sigma$, this takes the form
    \begin{align}
        L_{\text{SAM}} (\mu, \Sigma) \approx L(\mu) + \sqrt{p}\|\nabla_\theta L(\mu)\|_\Sigma = L(\mu) + \sqrt{p}\sqrt{\nabla^\top_\theta L(\mu) \Sigma \nabla_\theta L(\mu)} .
    \end{align}
\end{proposition}
\begin{proof}
Plugging the solution $\epsilon^*$ in (Eqn.\ \eqref{epsilon}) into the SAM objective and performing a first-order Taylor approximation, we arrive at
\begin{flalign*}
    L_{\text{SAM}} (\mu, \Sigma) 
    &= \text{max}_{\epsilon^T \Sigma^{-1} \epsilon < p} L(\mu + \epsilon)\\
    &= L(\mu + \epsilon^*)\\
    &= L(\mu) + \epsilon^{*\top} \nabla_\theta L(\mu) + O(\epsilon^{*\top}  \epsilon^*)\\
    &= L(\mu) + \sqrt{p} \frac{ \nabla^\top_\theta L(\mu) \Sigma}{\sqrt{\nabla^\top_\theta L(\mu) \Sigma \nabla_\theta L(\mu)}} \nabla_\theta L(\mu) + O(\epsilon^{*\top}  \epsilon^*) \\
    &= L(\mu) + \sqrt{p}\sqrt{\nabla^\top_\theta L(\mu) \Sigma \nabla_\theta L(\mu)}  + O(\epsilon^{*\top}  \epsilon^*),
\end{flalign*}
where we have used that $\Sigma$ is a symmetric matrix. 
In order to show that the error term is controlled, let us use express $\epsilon^*$ as in (Eqn.\ \eqref{reparam}) as $\epsilon^*=\Sigma^{\frac{1}{2}} \eta$. Then
\begin{align}
    \epsilon^{* \top}  \epsilon^* = \eta^{\top} \Sigma \eta.
\end{align}
From (Eqn.\ \eqref{reparam}), we see that $\| \eta \|_2 = \sqrt{p}$, which means that as long as the largest (in magnitude) eigenvalue of $\Sigma$, $\lambda_{\text{max}} < K \frac{1}{p}$ for some $K$ constant, we have $\eta^{\top} \Sigma \eta < K$.
In the special case of SAM, we have $\Sigma=\frac{\rho^2}{p}I$, hence the error term scales with $\rho^2$.
\end{proof}

\begin{proposition}
    Full batch gradient descent using the SAM step follows a path that is closest to the exact continuous path given by $\dot{\mu}=-\nabla_{\mu}\tilde{L}_{\text{SAM}}(\mu)$, where $\tilde{L}_{\text{SAM}}(\mu)$ is given by
    \begin{align}
        \tilde{L}_{\text{SAM}}(\mu)\approx L(\mu) + \sqrt{p} \|\nabla_\theta L(\mu)\|_\Sigma
        +\frac{\delta}{4}||\nabla_\theta L(\mu) ||_2^2,
    \end{align}
    where $\delta$ is the stepsize of the gradient descent algorithm.
\end{proposition}
\begin{proof}
    In the proof, we follow \citep{gdregul} in finding a modified loss surface, along which the exact path of gradient flow is closer to the discrete steps of gradient descent on the approximated gradient of the SAM objective (Eqn.\ \eqref{eqn:SAM_actual}). The general formula in \citep{gdregul} of the modified loss surface for loss function $E(\mu):=L_{\text{SAM}}(\mu)$ is
    \begin{align}
        \tilde{L}_{\text{SAM}}(\mu)={L}_{\text{SAM}}(\mu) +\frac{\delta}{4} ||\nabla_{\theta}{L}_{\text{SAM}}(\mu)||_2^2.
    \end{align}
    Now we can use the approximation in Proposition \ref{firstprop}. In the following, we concentrate on the general case (general $\Sigma$). For readability, we use the notation $\nabla_{\theta}L_{\text{SAM}}(\mu) = g(\mu)$.
    \begin{flalign}
    \label{sam_modified_flow}
    \tilde{L}_{\text{SAM}}(\mu)
    &= L(\mu) + \sqrt{p} \| g(\mu) \|_2 + \frac{\delta}{4} \Big| \Big| g(\mu) + \sqrt{p} \nabla_{\theta} \|g(\mu)\|_{\Sigma} \Big| \Big|_2^2.
    \end{flalign}
    We can approximate the rightmost term as
    \begin{align}
    \Big| \Big| g(\mu) + \sqrt{p} \nabla_{\theta} \|g(\mu)\|_{\Sigma} \Big| \Big|_2^2
    = \big(g^\top(\mu) + \sqrt{p} \nabla^\top_{\theta} \|g(\mu)\|_{\Sigma} \big) \big(g(\mu) + \sqrt{p} \nabla_{\theta} \|g(\mu)\|_{\Sigma}\big)
    \approx \|g(\mu)\|^2_2,
    \end{align}
    since the first two of the remaining terms scale with $\sqrt{p} \Sigma^{\frac{1}{2}}$, which can be made smaller than $\rho$ for sufficiently small $\Sigma$. The last remaining term scales with $p \Sigma$, which can be made smaller than $\rho^2$. Since we also have the scaling factor $\delta$ in (Eqn.\ \eqref{sam_modified_flow}), we may neglect these terms.
\end{proof}

\subsection{Implicit regularization of mean field variational inference}
\label{VOSharpness}
\begin{proposition}
    The following approximation holds for the RandomSAM with constant  $\Sigma=\frac{\rho^2}{p}I$:
    \begin{align}
        L_{\text{RSAM}}(\mu) 
        \approx L(\mu) + \frac{\rho^2}{2p} \text{Tr} H(\mu)
    \end{align}
where $H(\mu)$ denotes the Hessian. For MFVI with general (but symmetric positive definite) $\Sigma$ we have the following approximation:
    \begin{align}
        L_{\text{MFVI}}(\mu, \Sigma) 
        \approx L(\mu) + \frac{\text{Tr}\left[ \Sigma H(\mu)\right]}{2}
    \end{align}
\end{proposition}

\begin{proof}
Recall the objective of MFVI
\begin{align}
    L_{\text{MFVI}}(\mu)= L(\mu) + \big[ \mathbb{E}_{\eta \sim \mathcal{N}(0, I)} L(\mu+\Sigma^{\frac{1}{2}}\eta) - L(\mu) \big],
    \label{eqn:vo_eq}
\end{align}
where, similarly as in (Eqn.\ \eqref{eqn:SAM_ideal}), the term in brackets can be interpreted as a sharpness penalty. Note that we have fixed $\Sigma=\sigma^2 I$ for a constant $\sigma$. Using a second-order Taylor expansion around $\mu$, where $H(\mu)$ is the Hessian at $\mu$, 
\begin{flalign}
    L_{\text{MFVI}}(\mu)
    &\approx \mathbb{E}_{\eta \sim \mathcal{N}(0, I)} \big[ L(\mu) + \eta^\top (\Sigma^{\frac{1}{2}})^\top\nabla_{\theta}L(\mu) + \frac{1}{2} \eta^\top (\Sigma^{\frac{1}{2}})^\top H(\mu) \Sigma^{\frac{1}{2}} \eta \big]\\
    &= L(\mu) + \frac{1}{2} \mathbb{E}\Big[\text{Tr}(\eta^\top (\Sigma^{\frac{1}{2}})^\top H(\mu) \Sigma^{\frac{1}{2}} \eta)\Big]\\
    &= L(\mu) + \frac{1}{2} \mathbb{E}\Big[\text{Tr}(\eta \eta^\top (\Sigma^{\frac{1}{2}})^\top H(\mu) \Sigma^{\frac{1}{2}} )\Big] \\
    &= L(\mu) + \frac{1}{2}\text{Tr}\Big[\mathbb{E}(\eta \eta^\top) (\Sigma^{\frac{1}{2}})^\top H(\mu) \Sigma^{\frac{1}{2}} \Big] \\
    &= L(\mu) + \frac{1}{2}\text{Tr} \big[\Sigma H(\mu) \big].
\end{flalign}
If $\Sigma=\frac{\rho^2}{p} I$, we recover the first part of the theorem.
\end{proof}
Since the MFVI objective in (Eqn.\ \eqref{eqn:vo_eq}) is not typically available in closed form, in practice one uses a single-sample Monte Carlo estimate based on the reparametrization trick. This yields a stochastic objective as follows.
\begin{align}
    L_{\text{MFVI}}(\mu) \approx L(\mu + \Sigma^{\frac{1}{2}}\eta), \quad  \eta \sim \mathcal{N}(0, I)
\end{align}
We can follow the method of \citep{gdregul} to see that SGD on the single-sample Monte Carlo estimates of the gradient display additional implicit regularisation towards wider minima. This is carried out in Proposition \ref{mfvi_bwe}.

\begin{proposition}
    \label{mfvi_bwe}
    Gradient descent with MFVI step follows a path that is closest to gradient flow path on $\tilde{L}_\text{MFVI}(\mu, \Sigma)$, where $\tilde{L}_\text{MFVI}(\mu, \Sigma)$ is the following
    \begin{align}
        \tilde{L}_\text{MFVI}(\mu,\Sigma) &\approx \mathbb{E}_{\eta\sim\mathcal{N}(0,I)} \left[ L(\mu+\Sigma^{\frac{1}{2}}\eta) + \frac{\delta}{4} ||\nabla_\theta L(\mu + \Sigma^{\frac{1}{2}}\eta)||_2^2 \right] \\
        &\approx L(\mu) + \frac{\delta}{4} ||g(\mu)||_2^2 +\frac{\delta}{4} \text{Tr}[\Sigma H(\mu)^2]
    \end{align}
    where $g(\mu + \Sigma^\frac{1}{2}\eta)=\nabla_\theta L(\mu + \Sigma^\frac{1}{2} \eta).$

\end{proposition}

\begin{proof}

By \cite{gdregul}, we have the first approximation, i.e.
\begin{align}
\label{SmithUsed}
  \tilde{L}_\text{MFVI}(\mu,\Sigma) &\approx \mathbb{E}_{\eta\sim\mathcal{N}(0,I)} \left[ L(\mu+\Sigma^{\frac{1}{2}}\eta) + \frac{\delta}{4} ||\nabla_\theta L(\mu + \Sigma^{\frac{1}{2}}\eta)||_2^2 \right] 
\end{align}

Now with Taylor expansion we get
\begin{align}
\label{Taylor1}
    L(\mu + \Sigma^{\frac{1}{2}}\eta) = L(\mu) + (\Sigma^{\frac{1}{2}}\eta)^{\text{T}} g(\mu) + O(||\Sigma^{\frac{1}{2}}\eta||_2^2)
\end{align}
and 
\begin{align}
\label{Taylor2}
    g(\mu + \Sigma^{\frac{1}{2}}\eta) = g(\mu) + H(\mu)\Sigma^{\frac{1}{2}}\eta + O(||\Sigma^{\frac{1}{2}}\eta||_2^2)
\end{align}
Now by using (\ref{SmithUsed}), (\ref{Taylor1}) and (\ref{Taylor2}) we obtain
\begin{align}
    \tilde{L}_\text{MFVI}(\mu,\Sigma) &\approx \mathbb{E}_{\eta \sim \mathcal{N}(0,I)} \left[L(\mu) + (\Sigma^{\frac{1}{2}}\eta)^\top g(\mu) + \frac{\delta}{4} (g(\mu) + H(\mu)\Sigma^{\frac{1}{2}}\eta)^\top(g(\mu) + H(\mu)\Sigma^{\frac{1}{2}}\eta )  \right] \\
    &= L(\mu) + \frac{\delta}{4}\mathbb{E}_{\eta\sim\mathcal{N}(0,I)}\left[ g(\mu)^\top g(\mu) + 2g(\mu)^\top H(\mu)\Sigma^{\frac{1}{2}}\eta + (H(\mu)\Sigma^{\frac{1}{2}}\eta)^\top(H(\mu)\Sigma^{\frac{1}{2}}\eta) \right] \\
    &= L(\mu) + \frac{\delta}{4} g(\mu)^\top g(\mu) + \frac{\delta}{4} \mathbb{E}\left[\text{Tr}\left[(H(\mu)\Sigma^{\frac{1}{2}}\eta)^\top(H(\mu)\Sigma^{\frac{1}{2}}\eta) \right] \right] \\
    &=L(\mu) + \frac{\delta}{4} g(\mu)^\text{T}g(\mu) +  \frac{\delta}{4} \mathbb{E}\left[\text{Tr}\left[\eta\eta^\text{T}\Sigma^\frac{1}{2}H(\mu)^2\Sigma^\frac{1}{2}\right]\right] \\
    &= L(\mu) + \frac{\delta}{4} ||g(\mu)||_2^2 + \frac{\delta}{4} \text{Tr}\left[\mathbb{E}\left[\eta\eta^\text{T}\Sigma^\frac{1}{2}H(\mu)^2\Sigma^\frac{1}{2}\right]\right] \\
    &= L(\mu) + \frac{\delta}{4} ||g(\mu)||_2^2 + \frac{\delta}{4} \text{Tr}\left[\mathbb{E}\left[\eta\eta^\text{T}\right]\Sigma^\frac{1}{2}H(\mu)^2\Sigma^\frac{1}{2}\right] \\
    &= L(\mu) + \frac{\delta}{4} ||g(\mu)||_2^2 + \frac{\delta}{4} \text{Tr}\left[\Sigma H(\mu)^2\right].
\end{align}
\end{proof}

This shows that Mean Field Variational Inference implicitly regularizes the trace of the square of the Hessian, and hence it also regularizes the magnitude of its eigenvalues. A way of defining sharpness is via the local curvature of the loss function around the minimum given that it is a critical point \citep{Dinh17}. Since local curvature is encoded in the Hessian eigenvalues, this means that MFVI penalizes a notion of sharpness at critical points of the loss landscape.


\subsection{Information theoretic motivation for VariationalSAM}
\label{VSAMMot}
The SAM, ASAM and FSAM algorithms all modify the loss function to penalise the maximum loss value within a small neighbourhood around the current weights. The considered neighbourhood is an Euclidean ball in SAM, a weight dependent ellipsoid in ASAM, and an ellipsoid defined by the Fisher information matrix in FSAM. What we suggest is a generalization of these approaches: our method omits any constraint on the ellipsoid defining the neighbourhood, and treats it as an object to learn. Specifically, besides $\mu$ the algorithm also optimizes a symmetric positive-definite matrix $\Sigma$, because for such matrix $\epsilon^\top \Sigma^{-1} \epsilon \leq p$ describes an ellipsoid.

Following the ideas of the previously mentioned algorithms, the modified loss function would be
\begin{equation}
L_{\text{VSAM}}(\mu, \Sigma)= \max_{\epsilon ^\top \Sigma^{-1} \epsilon  \leq p} L(\mu +\epsilon).
\end{equation}
Our intention is to do coordinate descent on $L_{\text{VSAM}}(\mu, \Sigma)$ w.r.t. $\mu$ and $\Sigma$. However, minimizing the loss in $\Sigma$ would lead to the null matrix, which is undesirable, therefore we would like to add additional terms to the loss. The idea is motivated by the Variational Bayesian methods, specifically the maximization of the Evidence Lower Bound (ELBO) and equivalently, the minimization of the negative ELBO. We have $X$ random variable (our observable date) and $Z$ random variable, that is the parameter of our function wich generates the output. Consider a prior $p(Z)$, a likelihood $p(X|Z)$, and an arbitrary distribution $q_{\theta}(Z)$, then the posterior $p(Z|X)$ and the evidence $p(X)$ can be computed, and the minimization of the negative ELBO has the form of
\begin{equation}
\begin{split}
    \argmin_{\theta} - \mathcal{L}_{ELBO} &= \argmin_{\theta} \text{KL}(q_{\theta}(Z)||p(Z|X)) - \log p(X) \\
    &= \argmin_{\theta} \mathbb{E}_{q_{\theta}}\left[ - \log p(X|Z)\right] + \text{KL}(q_{\theta}||p)
\end{split}
\end{equation}
When the loss function is defined as the negative log-likelihood, the expectation can be considered as the expectation of the loss (note, that the loss is averaged over our data points, so we get a $1/N$ term in front of the KL-divergence).
\begin{equation}
 \argmin_{\theta} - \mathcal{L}_{ELBO} = \argmin_{\theta} \mathbb{E}_{z \sim q_{\theta}} \left[ L(z)\right] + \frac1N \text{KL}(q_{\theta}||p)
\end{equation}
Choosing $p(Z) = \mathcal{N}(0, \sigma_0 ^2 I)$ and $q_{\theta}(Z) = \mathcal{N}(\mu,  \Sigma)$ as $k$-dimensional Gaussians, the KL-divergence can be rewritten as
\begin{equation}
    \text{KL}\big[\mathcal{N}(\mu, \Sigma)||\mathcal{N}(0, \sigma_0 ^2 I)\big]= \frac12 \Bigg[\frac{1}{\sigma_0^2} \mathrm{Tr} \Sigma + \log \det \Sigma^{-1} + \log \sigma_0^{2k} + \frac{1}{\sigma_0^2} \| \mu \| ^2-p\Bigg].
\end{equation}
By slightly rephrasing $\mathbb{E}_{z \sim q_{\theta}} L(z)$ as $\mathbb{E}_{\epsilon \sim \mathcal{N}(0, \Sigma)} L (\mu + \epsilon)$ we get
\begin{equation}
    \argmin_{\theta} - \mathcal{L}_{ELBO} = \argmin_{\theta} \mathbb{E}_{\epsilon \sim \mathcal{N}(0,  \Sigma)} L (\mu + \epsilon) +\frac{1}{2N \sigma_0^2} \mathrm{Tr} \Sigma + \frac{1}{2N} \log \det \Sigma^2  + \frac{1}{2N \sigma_0^2} \| \mu \| ^2
\end{equation}
The expectation above can be bounded with $\max_{\epsilon ^\top \Sigma^{-1} \epsilon  \leq p} L(\mu +\epsilon)$ in a similar way to the derivations in SAM and FSAM, if $\rho$ is sufficiently large. This motivates the minimization of

\begin{equation}
\label{pac_bound}
    \max_{\epsilon ^\top \Sigma^{-1} \epsilon  \leq p} L(\mu +\epsilon) +  \frac{1}{2N \sigma_0 ^2} \mathrm{Tr} \Sigma + \frac{1}{2N} \log \det \Sigma^{-1} + \frac{1}{2N \sigma_0 ^2} \| \mu \| ^2.
\end{equation}
%

We note that the derivations of SAM, ASAM, FSAM also use the same approach, but they bound the KL term further. Instead, we found these terms interesting to keep. 

Using Appendix \ref{ImplicitRegularization}, we can approximate the loss as
\begin{equation}
    L_{\text{VSAM}}(\mu, \Sigma) = L(\mu) + \sqrt{p} \sqrt{\nabla ^\top l(\mu) \Sigma \nabla L(\mu)} +  \frac{1}{2N \sigma_0 ^2} \mathrm{Tr} \Sigma + \frac{1}{2N} \log \det \Sigma^{-1} + \frac{1}{2N \sigma_0 ^2} \| \mu \| ^2.
\end{equation}
We take gradient descent steps w.r.t. $\mu$ and $\Sigma$ alternately. 

\subsection{PAC-Bayes bound for VariationalSAM}
\label{Pac}
\begin{theorem}
For a parameter space $\mathcal{M}\times\mathcal{N}$ (with later described properties) and for any $(\mu,\Sigma)\in \mathcal{M}\times\mathcal{N}$ we have with probability at least $1-\delta$, that

\begin{align}
    \mathds{E}_{\epsilon\sim \mathcal{N}(0,\Sigma)}[L_D(\mu+\epsilon)]&\leq \text{max}_{\epsilon^T\Sigma^{-1}\epsilon\leq \gamma^2} L_S(\mu+\epsilon)+ \sqrt{\frac{O(p+\log\frac{n}{\delta})}{n-1}}
\end{align}

where $L_D$ is the generalization loss, and $ \text{max}_{\epsilon^T\Sigma^{-1}\epsilon\leq \gamma^2} L_S(\mu+\epsilon)$ is the empirical SAM loss, with the geometry provided by $\Sigma$, i.e. the VSAM loss without the KL divergence term and $\gamma=p(1+\sqrt{\log n/p})$

\end{theorem}

\begin{proof}
To build our PAC-Bayes bound we follow the steps laid out in  \cite{FSAM}. Let us take the parameter spaces $\mu\in\mathcal{M}\subset\mathbb{R}^p$ and $\Sigma\in\mathcal{N}$, where $\mathcal{N}$ is the subset of the diagonal, positive definite matrices. Also assume, that $\text{diam}(\mathcal{M})\leq M$ and $1/\lambda\leq\Sigma_{k,k}\leq \lambda$ for any $\Sigma\in\mathcal{N}$ and $k$, with some constant $\lambda>1$.  We will take the set of ellipsoids $P_{1,1},\cdots, P_{t,s}$, where

\begin{align}
\label{set}
    P_{i,j} = \{(\mu,\Sigma)\in\mathcal{M}\times\mathcal{N}:&(\mu-\mu_i)^T\Sigma_j^{-1}(\mu-\mu_i)\leq r^2 \text{ and } \frac{[\Sigma_j]_{k,k}}{\Sigma_{k,k}}=1+c_k, \\
    &\text{where }c_k\in[-c,c]\quad\forall k\}
\end{align}

There exists a finite set of these ellipsoids, that cover $\mathcal{M}\times\mathcal{N}$.

\begin{enumerate}
    \item Let us take the following subset of $\mathbb{R}^p$
    \begin{align*}
        P_{i,j}|_\mu = \{\mu:\mu\in P_{i,j}\} \\
        vol(P_{i,j}|_\mu)\propto r^p \cdot \Sigma_j^{\frac{1}{2}} 
    \end{align*}
    thus $t=O(M^p/r^p)$, so $\log(t)=O(p)$.
    \item We also assumed $\frac{1}{\lambda}\leq \Sigma_{k,k}\leq\lambda$. Let's think about $\mathcal{N}$ as a subset of $\mathbb{R}^p$. Then as in part (1) we can define
    \begin{align*}
        P_{i,j}|_\Sigma = \{\Sigma:\Sigma\in P_{i,j}\}\\
        vol(P_{i,j}|_\Sigma) \propto (2c)^p
    \end{align*}
    thus $s=O((\lambda - \frac{1}{\lambda})^p/(2c)^p)$, so $\log(s)=O(p)$
\end{enumerate}

From this point the proof is the same as in \cite{FSAM}, so many details are omitted. Now we use the PAC-Bayes Theorem from \cite{McAllester} as follows. For any prior distribution $P(\mu,\Sigma)$, posterior distribution $Q(\mu,\Sigma)$ and training set $\mathcal{N}$ we have with probability at least $1-\delta$

\begin{align}
\label{PAC}
    \mathds{E}_{Q(\mu,\Sigma)}[L_D(\mu,\Sigma)] \leq \mathds{E}_{Q(\mu,\Sigma)}[L_S(\mu,\Sigma)] + \sqrt{\frac{\text{KL}[Q(\mu,\Sigma)||P(\mu,\Sigma)] + \log\frac{n}{\delta}}{2(n-1)}}
\end{align}

\noindent $L_D(\mu,\Sigma)$ and $L_S(\mu,\Sigma)$ are generalization and empirical losses. We choose $Q(\mu, \Sigma)=\mathcal{N}(\mu_0,\Sigma_0)$, where $(\mu_0,\Sigma_0)\in\mathcal{M}\times\mathcal{N}$ and choose our set of priors as $P_{1,1},\cdots,P_{t,s}$, where $P_{i,j}=\mathcal{N}(\mu_i,\Sigma_j)$. Then as in (\ref{PAC}) for each $i,j$ we have with probability $1-\delta_{i,j}$, that

\begin{align}
\label{PACij}
    \forall Q(\mu,\Sigma) \quad \mathds{E}_{Q(\mu,\Sigma)}[L_D(\mu,\Sigma)]\leq \mathds{E}_{Q(\mu,\Sigma)}[L_S(\mu,\Sigma)] + \sqrt{\frac{\text{KL}[Q(\mu,\Sigma)||P_{i,j}(\mu,\Sigma)]+\log\frac{n}{\delta_{i,j}}}{2(n-1)}}
\end{align}

In the intersection (\ref{PACij}) holds for every $P_{i,j}$. So by Union bound theorem the intersection is at least $1-\sum_{i,j} \delta_{i,j}$, so if we take $\delta_{i,j}=\frac{\delta}{st}$, we have with probability at least $1-\delta$

\begin{align}
    \label{union}
    \forall Q(\mu,\Sigma) \quad  \mathds{E}_{Q(\mu,\Sigma)}[L_D(\mu,\Sigma)]\leq \mathds{E}_{Q(\mu,\Sigma)}[L_S(\mu,\Sigma)] + \\
    + \sqrt{\frac{\text{KL}[Q(\mu,\Sigma)||P_{i,j}(\mu,\Sigma)]+\log\frac{n}{\delta} + \log(s) + \log(t) }{2(n-1)}} \quad \forall i,j
\end{align}

If we choose the prior closes to $Q(\mu,\Sigma)$ we get the following:

\begin{align}
    \text{KL}[Q||P_{i,j}] = \frac{1}{2} \left( \text{Tr}(\Sigma_j\Sigma_0^{-1}) + (\mu_0-\mu_i)^T\Sigma_j^{-1}(\mu_0-\mu_i) + \log \frac{|\Sigma_0|}{|\Sigma_j|} - p \right)
\end{align}

From our assumptions $\text{Tr}(\Sigma_j\Sigma_0^{-1})\leq p(1+c)$, $\log\frac{|\Sigma_0|}{|\Sigma_j|} \leq \sum_k log(1+c_k) \leq \sum_k c_k \leq pc$ and $(\mu_0-\mu_i)^T\Sigma_j^{-1}(\mu_0-\mu_i)\leq r^2$. Thus we get

\begin{align}
    \text{KL}[Q||P_{i,j}]\leq \frac{1}{2}(p+pc+r^2+pc-p)=pc+\frac{r^2}{2}
\end{align}

Let us rephrase (\ref{union}) as we plug in $\mu + \epsilon$ instead of $\mu$ where $\epsilon \sim \mathcal{N}(0,\Sigma)$. Thus we get:

\begin{align}
    \label{boundgood}
    \forall \mu,\Sigma \quad \mathds{E}_{\epsilon\sim \mathcal{N}(0,\Sigma)}[L_D(\mu+\epsilon)] \leq \mathds{E}_{\epsilon\sim \mathcal{N}(0,\Sigma)}[L_S(\mu+\epsilon)] + \sqrt{\frac{pc+r^2/2 + \log\frac{n}{\delta} + \log(s) +\log(t)}{2(n-1)}}
\end{align}

Let $u=\Sigma^{-1/2}\epsilon$, so $u\sim \mathcal{N}(0,1)$ using the result from  \cite{Laurent&Massart2000} we get that with probability at least $1-\frac{1}{\sqrt{n}}$

\begin{align}
    \label{ineq}
    |u|_2^2=\epsilon^T\Sigma^{-1}\epsilon\leq p(1+\sqrt{\log n/p})^2=\gamma^2
\end{align}

Let's partition the space into two parts. One where (\ref{ineq}) holds, where we take the maximum loss and on where (\ref{ineq}) does not hold, where we choose the loss bound $L_\text{max}$. So we have

\begin{align}
    \label{bound}
    \mathds{E}_{\epsilon\sim \mathcal{N}(0,\Sigma)}[L_S(\mu+\epsilon)]&\leq (1-1/\sqrt{n})\text{max}_{\epsilon^T\Sigma^{-1}\epsilon\leq \gamma^2} L_S(\mu+\epsilon) + \frac{L_\text{max}}{\sqrt{n}} = \\
    &=\text{max}_{\epsilon^T\Sigma^{-1}\epsilon\leq \gamma^2} L_S(\mu+\epsilon) + \frac{L_\text{max}}{\sqrt{n}}
\end{align}

Plugging (\ref{bound}) into (\ref{boundgood}) we get:

\begin{flalign}
    \mathds{E}_{\epsilon\sim \mathcal{N}(0,\Sigma)}[L_D(\mu+\epsilon)]&\leq \text{max}_{\epsilon^T\Sigma^{-1}\epsilon\leq \gamma^2} L_S(\mu+\epsilon) + \frac{L_\text{max}}{\sqrt{n}}+\\ + \sqrt{\frac{\frac{r^2(\sqrt{p}+\sqrt{\log n})^2}{2\gamma^2} + pc + \log\frac{n}{\delta} +\log(st) }{2(n-1)}} &=  \text{max}_{\epsilon^T\Sigma^{-1}\epsilon\leq \gamma^2} L_S(\mu+\epsilon)+ \sqrt{\frac{O(p+\log\frac{n}{\delta})}{n-1}}
\end{flalign}
\end{proof}

Note, that similar bounds can be built for Random SAM and MFVI.


\end{document}